%% file: main.tex
\documentclass[11pt,letterpaper]{article}

\usepackage[utf8]{inputenc} 
\usepackage[T1]{fontenc}    
\usepackage{hyperref}       
\usepackage{url}            
\usepackage{booktabs}       
\usepackage{xcolor}
\usepackage{amsfonts}       
\usepackage{nicefrac}       
\usepackage{microtype}      
\usepackage{appendix}
\usepackage{amsthm}
\usepackage{amsmath}
\usepackage{amssymb}
\usepackage{enumerate}
\usepackage{algorithm}
\usepackage{algorithmic}
\usepackage{geometry}
\usepackage{setspace}
\usepackage{xspace}
\usepackage{mathabx}
\usepackage{graphicx}
\usepackage{times}
\usepackage{tabularx, makecell, booktabs, array, longtable}
\usepackage{natbib}
\usepackage{verbatim}

\input{_includes}

\makeatletter
\newif\ifhf@appendixtoc
\newcommand{\hf@appendixtocline}[2]{%
  \addtocontents{atoc}{\protect\contentsline{#1}{#2}{\thepage}{\@currentHref}\protected@file@percent}%
}
\newcommand{\hf@maybeappendixtocline}[2]{%
  \def\hf@entrytype{#1}%
  \def\hf@sectiontype{section}%
  \def\hf@subsectiontype{subsection}%
  \def\hf@subsubsectiontype{subsubsection}%
  \ifx\hf@entrytype\hf@sectiontype
    \hf@appendixtocline{#1}{#2}%
  \else\ifx\hf@entrytype\hf@subsectiontype
    \hf@appendixtocline{#1}{#2}%
  \else\ifx\hf@entrytype\hf@subsubsectiontype
    \hf@appendixtocline{#1}{#2}%
  \fi\fi\fi
}
\let\hf@oldaddcontentsline\addcontentsline
\renewcommand{\addcontentsline}[3]{%
  \hf@oldaddcontentsline{#1}{#2}{#3}%
  \ifhf@appendixtoc
    \def\hf@contentsfile{#1}%
    \def\hf@tocfile{toc}%
    \ifx\hf@contentsfile\hf@tocfile
      \hf@maybeappendixtocline{#2}{#3}%
    \fi
  \fi
}

\newcommand{\startappendixcontents}{\hf@appendixtoctrue}
\newcommand{\appendixtableofcontents}{%
  \section*{\contentsname}%
  \@starttoc{atoc}%
}
\makeatother

\title{Efficient Sequential Calibration with 
$O(T^{2/3-\epsilon})$ Error Bound}
\author{Zihan Zhang \\ Department of CSE, HKUST \\ zihanz@cse.ust.hk}
\date{}

\begin{document}
\maketitle
\thispagestyle{empty}

\begin{abstract}
We study the online binary sequential calibration problem. A recent 
breakthrough by \citet{dagan2024breaking} overcomes the
classical \(T^{2/3}\) barrier for calibration error. Building on this result,  we present an efficient
randomized forecaster that achieves an
expected calibration error \(O(T^{2/3-\varepsilon})\) for some constant
\(\varepsilon>0\). 

Our forecaster combines the \textsc{SPR-Calibration} procedure  \citep{dagan2024breaking} with an outer Blackwell-style correction
layer. The \textsc{SPR-Calibration} procedure controls calibration with respect to a surrogate
sequence of conditional-mean estimates, while the correction layer controls the
additional error incurred when these surrogates are used to approximate the true outcomes. The analysis decomposes the total calibration error into the
surrogate calibration error and the residual discrepancy between the surrogate
sequence and the true outcomes. The former is bounded by the \textsc{SPR-Calibration} guarantee in \citet{dagan2024breaking}, and
the latter is controlled using a quadratic potential argument together with the
sparsity of the \textsc{SPR-Calibration} forecaster. 
\end{abstract}


\input{intro}

\input{setting}

\input{SPR}
\input{algorithm}

\input{analysis}

\input{discussion}

\bibliography{ref}
\bibliographystyle{apalike}

\end{document}

%% file: _includes.tex
\usepackage{amsthm, amsmath, amssymb, mathabx}
\allowdisplaybreaks 
\usepackage{graphicx}
\usepackage{enumerate}
\usepackage{pifont}
\usepackage{booktabs}
\usepackage{multirow}
\usepackage{multicol}
\usepackage{stfloats}
\usepackage{xspace}
\usepackage{thmtools}
\usepackage{thm-restate}
\usepackage{hyperref}
\usepackage{cleveref}
\crefname{appendix}{appendix}{appendices}
\Crefname{appendix}{Appendix}{Appendices}
\crefname{subappendix}{appendix}{appendices}
\Crefname{subappendix}{Appendix}{Appendices}
\crefname{subsubappendix}{appendix}{appendices}
\Crefname{subsubappendix}{Appendix}{Appendices}
\makeatletter
\let\hf@oldappendix\appendix
\renewcommand{\appendix}{%
  \hf@oldappendix
  \crefalias{section}{appendix}%
  \crefalias{subsection}{subappendix}%
  \crefalias{subsubsection}{subsubappendix}%
}
\makeatother
\crefname{ALC@unique}{line}{lines}
\Crefname{ALC@unique}{Line}{Lines}
\makeatletter
\renewcommand{\theALC@unique}{\arabic{ALC@line}}
\renewcommand{\theHALC@unique}{\arabic{ALC@unique}}
\makeatother
\usepackage{anyfontsize} 
\usepackage{makecell}
\usepackage{hhline}
\usepackage{colortbl}
\usepackage{xpatch}
\usepackage{scalerel,stackengine}
\usepackage{sidecap}

\makeatletter
\def\@fnsymbol#1{\ensuremath{\ifcase#1\or *\or \dagger\or \ddagger\or
  \mathsection\or \mathparagraph\or \|\or \diamond \or **\or \dagger\dagger
  \or \ddagger\ddagger \else\@ctrerr\fi}}
\makeatother

\makeatletter
\newcommand{\printfnsymbol}[1]{%
  \textsuperscript{\@fnsymbol{#1}}%
}
\makeatother

\newtheorem{theorem}{Theorem}
\newtheorem{lemma}{Lemma}

\newtheorem{definition}{Definition}

\crefname{condition}{Condition}{Conditions}

\crefname{assumption}{Assumption}{Assumptions}

\theoremstyle{definition}

\newcommand{\abs}[1]{\left| #1 \right|}

\newcommand{\Roma}[1]{\uppercase\expandafter{\romannumeral#1}}

\usepackage[colorinlistoftodos, textwidth=18mm]{todonotes}

\def\shownotes{1}
\ifnum\shownotes=0
\newcommand{\todorz}[1]{}
\newcommand{\todorzout}[1]{}
\newcommand{\todossdout}[1]{}
\newcommand{\todossd}[1]{}
\else
\newcommand{\todorz}[1]{\todo[color=blue!10, inline]{\small RZ: #1}}
\newcommand{\todorzout}[1]{\todo[color=blue!10]{\scriptsize RZ: #1}}
\newcommand{\todossdout}[1]{\todo[color=red!10]{\scriptsize SSD: #1}}
\newcommand{\todossd}[1]{\todo[color=red!10, inline]{\small SSD: #1}}
\fi

%% file: intro.tex
\section{Introduction}\label{sec:intro}

Calibration captures a basic reliability requirement for probabilistic
predictions: when a forecaster repeatedly assigns probability \(p\) to an
event, the event should occur on roughly a \(p\)-fraction of those rounds. This
notion is important because probability forecasts are often used directly in
downstream decisions, such as risk assessment, medical prediction, weather
forecasting, and machine-learning systems that report confidence scores
\citep{dawid1982wellcalibrated,guo2017calibration,kuleshov2018accurate,hebert2018multicalibration}.
In these applications, the numerical value of a forecast matters, not only the
ranking of alternatives: a prediction of \(0.8\) should be interpretable as an
event that happens about \(80\%\) of the time. Sequential calibration studies
how to achieve this reliability guarantee when predictions are made online and
the data may arrive adaptively or non-stationarily.

In this work, we consider the online sequential calibration problem initiated by \citet{foster1998asymptotic}.
At each round $t$, a forecaster announces a probability \(p_t\in[0,1]\) for a
binary outcome \(y_t\in\{0,1\}\).
The standard
\(\ell_1\)-calibration error after \(T\) rounds is
\[
\operatorname{CalErr}_T
=
\sum_{p\in\mathcal P_T}
\left|
\sum_{t\le T:p_t=p}(p-y_t)
\right|,
\]
where \(\mathcal P_T\) is the set of distinct predictions used by the
forecaster. Classical algorithms achieve calibration error of order
\(T^{2/3}\). One way to obtain this rate is through a discretization and
Blackwell-approachability argument \citep{foster1999blackwell}: the forecaster restricts predictions to a
finite grid of size \(m\), controls the calibration residuals over this grid
using an approachability strategy, and balances the approachability term
\(O(\sqrt{mT})\) with the discretization term \(O(T/m)\). Optimizing over
\(m\) gives \(m\asymp T^{1/3}\), and hence calibration error
\(O(T^{2/3})\). This approach is conceptually clean and computationally
efficient, but it also highlights the longstanding \(O(T^{2/3})\) barrier.

A recent breakthrough by \citet{dagan2024breaking} showed that the classical
\(O(T^{2/3})\) upper bound is not intrinsic. They introduced an algorithm, which
we refer to as \textsc{SPR-Calibration}, that achieves a calibration error bound
of order \( O(T^{2/3-\varepsilon})\) for some constant \(\varepsilon>0\), thereby
breaking the \(T^{2/3}\) barrier for the first time.

Despite this breakthrough, obtaining an efficient \( O(T^{2/3-\varepsilon})\)
guarantee in the sequential setting is not immediate.
The \textsc{SPR-Calibration} algorithm  \citep{dagan2024breaking} is analyzed
through a minimax reduction that allows the proof to work in a full-information
model, where the forecaster is given the conditional mean
\(e_t=\mathbb E[y_t\mid \mathcal F_{t-1}]\) at each round. This idea goes back to
the minimax proof of calibration~\citep{hart2022calibrated}: if the forecaster
knew the adversary's mixed strategy, then it could simply predict the induced
conditional probability of the next outcome; the minimax theorem then converts
this observation into the existence of a randomized forecasting strategy that is
calibrated against every adversary. However, this minimax transformation is primarily an existence argument rather than a computationally tractable algorithm. 

It then naturally raises the question:
\begin{center}
\emph{\textbf{Is there  an efficient algorithm with $O(T^{2/3-\varepsilon})$ calibration error?  }}
\end{center}

In this work, we answer this problem affirmatively.

\begin{theorem}\label{thm:main}
There exists a forecaster (see Algorithm~\ref{alg:main}), such that for any  sequence $\{y_t\}_{t=1}^T\in \{0,1\}^T$, the expected calibration error $\mathbb{E}[\operatorname{CalErr}_T]$ is bounded by $O(\log_2(T)\cdot T^{2/3 - \varepsilon_{\mathrm{AB}}/18 })$,
where $\varepsilon_{\mathrm{AB}}>0$ is the same as the constant $\varepsilon$ in Theorem 1.3 of \citet{dagan2024breaking}. Moreover, the computation cost of the algorithm is $O(T^{7/2}\log_2(T))$.
\end{theorem}

Our algorithm is based a natural combination of the
\textsc{SPR-Calibration} procedure of~\citet{dagan2024breaking} with an outer
Blackwell-style correction layer.   At each round \(t\), the algorithm first forms a surrogate estimate
\(\widetilde e_t\) of the conditional mean and then passes this value to the
calibration procedure of~\citet{dagan2024breaking}. Since calibration is ultimately
measured against the realized outcome \(y_t\), using the surrogate
\(\widetilde e_t\) introduces an additional source of error. To control this
error, the algorithm augments the underlying calibration procedure with a
Blackwell-style correction layer. The analysis separates the total calibration
error into two components: the calibration error with respect to the surrogate
sequence \(\{\widetilde e_t\}_{t=1}^T\), and the discrepancy between the surrogate
means and the realized outcomes. The first component is controlled directly by
the guarantee of~\citet{dagan2024breaking}; the second is bounded using a
quadratic potential argument together with the sparsity property of their
calibration algorithm.

\subsection{Related Works}\label{sec:rel}

\paragraph{Calibration and adversarial forecasting.}
Calibration has a long history as a criterion for evaluating probabilistic
forecasts. \citet{dawid1982wellcalibrated} emphasized calibration as a
basic consistency requirement for subjective probabilities, and \citet{foster1998asymptotic} initiated the adversarial sequential
calibration problem studied in this paper. Their work showed that randomized
forecasters can be calibrated against arbitrary binary outcome sequences and
gave the classical \(O(T^{2/3})\) calibration-error guarantee. Several
subsequent works gave alternative proofs and perspectives on calibration,
including the myopic minimax construction \citep{fudenberg1999easier}, the Blackwell-approachability proof 
\citep{foster1999blackwell}, the minimax proof 
\citep{hart2022calibrated}, and the geometric approachability proof 
\citep{mannor2010geometric}.

\paragraph{Approachability, regret, and calibration.}
Blackwell's approachability theorem~\citep{blackwell1956analog} is one of the
central tools behind adversarial calibration. In the approachability
formulation, calibration residuals are treated as coordinates of a
vector-valued payoff, and the forecaster chooses predictions so that the
cumulative payoff approaches an appropriate target set. This connection was
made explicit by \citet{foster1999blackwell} and was further developed
through geometric and online-learning viewpoints~\citep{mannor2010geometric,
abernethy2011blackwell}. Calibration is also closely related to regret
minimization: \citet{foster1999regret} connected calibration
and internal regret, and calibrated learning rules are known to lead to
correlated equilibrium in repeated games~\citep{foster1997calibrated}. The
broader connections between prediction of individual sequences, regret, and
game-theoretic learning are surveyed in~\citet{cesa-bianchi2006prediction}.

\paragraph{Rates for sequential calibration.}
The optimal rate of \(\ell_1\)-calibration error has been a central question in
sequential calibration. The classical upper bound is \(O(T^{2/3})\), obtained
by balancing the discretization error of a finite grid with an
approachability-type residual term~\citep{foster1998asymptotic,
abernethy2011blackwell}. For many years, the only general lower bound was the
trivial \(\Omega(\sqrt T)\) bound obtained from independent fair coin flips.
\citet{qiao2021stronger} gave the first super-\(\sqrt T\)
lower bound, proving an \(\Omega(T^{0.528})\) lower bound via the
sign-preservation game. \citet{dagan2024breaking} recently broke
the \(T^{2/3}\) upper-bound barrier by introducing \emph{sign preservation with reuse} 
(\textsc{SPR}), proving that improved \textsc{SPR} strategies yield calibration error
\(O(T^{2/3-\varepsilon})\) for some constant \(\varepsilon>0\). They also
improved the lower bound to \(\Omega(T^{0.54389})\), leaving a gap between the
known upper and lower exponents.

%% file: setting.tex
\section{Problem Setting}\label{sec:p_setting}

We follow the sequential calibration setup of \citet{dagan2024breaking}.
Fix a time horizon \(T\). At each round \(t\in[T]\), the forecaster
outputs a probability prediction \(p_t\in[0,1]\), and the environment outputs an
outcome \(y_t\in\{0,1\}\). The prediction \(p_t\) is interpreted as the
forecaster's announced probability that \(y_t=1\).

The adversary may be adaptive, in the sense that its choice at time \(t\)
may depend on the past history
\[
H_t:=\{(p_s,y_s):s<t\},
\]
but it does not observe the current prediction \(p_t\) before choosing
\(y_t\).

Let $
\mathcal P_t:=\{p_s:s\le t\}$ 
denote the set of distinct probability values actually output up to time
\(t\). Define  $
E_t(p)
=
\sum_{s\le t:p_s=p}(p-y_s).$
The cumulative \(\ell_1\)-calibration error at time \(t\) is
\[
\operatorname{CalErr}_t
:=
\sum_{p\in\mathcal P_t}|E_t(p)|
=
\sum_{p\in\mathcal P_t}
\left|
\sum_{s\le t:p_s=p}(p-y_s)
\right|.
\]

When the forecaster is randomized, all quantities above are random variables.
We measure performance by expected calibration error $
\mathbb E[\operatorname{CalErr}_T]$, 
where the expectation is over the forecaster's internal randomness.

We next recall the sign-preservation with reuse game, abbreviated as
\textsc{SPR}, introduced in \citet{dagan2024breaking}. 
\begin{definition}[Sign-preservation with reuse]
\label{def:spr-game}
For \(n,s\in\mathbb N\), the game
\(\operatorname{SPR}(n,s)\) is played between two players, called
Player-P and Player-L. The game consists of \(n\) cells, indexed by
\([n]=\{1,\ldots,n\}\), all of which are initially empty. The game lasts for at
most \(s\) rounds. In each round, the following steps occur:
\begin{enumerate}
    \item Player-P may terminate the game. Otherwise, Player-P chooses an empty
    cell \(j\in[n]\).

    \item After observing \(j\), Player-L may remove any subset of the
    \(-\) signs in cells strictly to the left of \(j\), and any subset of the
    \(+\) signs in cells strictly to the right of \(j\).

    \item Player-L then places either a \(+\) sign or a \(-\) sign in cell
    \(j\).
\end{enumerate}
A cell whose sign has been removed becomes empty and may therefore be chosen
again in a later round. Player-P aims to maximize the number of signs remaining
on the board at the end of the game, while Player-L aims to minimize this
quantity.
\end{definition}

%% file: SPR.tex
\section{The \textsc{SPR-Calibration} Algorithm}\label{sec:boundref}

In this section, we recall the \textsc{SPR-Calibration} construction of \cite{dagan2024breaking} and the properties needed in our analysis. 

\input{SPR_alg}

\subsection{ Theoretical Guarantees of \textsc{SPR-Calibration}}\label{sec:spreb}
We recall several properties of \textsc{SPR-Calibration} from
\cite{dagan2024breaking}.  The half-open interval convention used in
our presentation only makes the assignment of dyadic boundary points
explicit.  It preserves the containment, ordering, and dyadic counting
properties used in the cited proofs, and therefore does not affect the
following bounds.

Let \(\mathsf L^{\mathrm{AB}}\) be the deterministic Player--L
strategy implemented by Algorithms~\ref{alg:dagan-A}
and~\ref{alg:dagan-B}, with the root \(A\)-instance initialized on
the \(n\)-cell board with bias parameter \(b=0\).
 We write $\mathsf{SimAB}_n$ for
the composite transition implemented by \texttt{simulateGame}: on a
legal query to an empty cell, it first performs the \textsc{SPR}
sign-removal step, then advances the $A/B$ labeling state, and finally
places the sign returned by that state.

A \emph{full transcript} of this composite process is a finite sequence $
    H=\bigl((t_r,c_r,\sigma_r)\bigr)_{r=1}^{m}$, 
where $t_1<\cdots<t_m$ are the ambient calibration rounds,
$c_r\in[n]$ is empty immediately before the $r$-th call, and
$\sigma_r\in\{+,-\}$ is the sign returned by the evolving $A/B$ state.

For a signed call $(c,\sigma)$ and a later queried cell $c'$, define
\[
    \operatorname{kill}\bigl((c,\sigma),c'\bigr)
    :=
    \begin{cases}
        1,
        & \sigma=- \text{ and } c<c',\\
        1,
        & \sigma=+ \text{ and } c>c',\\
        0,
        & \text{otherwise}.
    \end{cases}
\]

Thus,
$\operatorname{kill}((c,\sigma),c')=1$ when the sign-removal step of a call to $c'$ erases the sign
$\sigma$ previously placed at $c$.

\begin{definition}[Reduced transcript]
\label{def:reduced-transcript}
Set $R_0:=\emptyset$.  Having constructed $R_{r-1}$, repeatedly delete the
last entry $(t,c,\sigma)$ of the current list while $
    \operatorname{kill}\bigl((c,\sigma),c_r\bigr)=1$,
and then append $(t_r,c_r,\sigma_r)$.  Let $    \operatorname{Red}(H):=R_m$ and $
    \operatorname{rlen}(H)
    :=\abs{\operatorname{Red}(H)}.$
The signs in $\operatorname{Red}(H)$ are inherited from the full transcript. This is a deterministic realization
of the adjacent-deletion reduction used in the proof of Lemma~A.2 of
\cite{dagan2024breaking}.
\end{definition}

By recording the \textsc{SPR} state associated with every prefix of the
reduced transcript, the A/B labeling procedure on reduced transcripts is
deterministic and well defined. More precisely, at step $\tau$, let
$R_{\tau-1}$ be the current reduced transcript and let $c_\tau$ be the incoming
cell. We initialize a temporary transcript $\widetilde{R}\gets R_{\tau-1}$.
When $\widetilde{R}$ is nonempty and its last entry $(t,c,\sigma)$ is killed by
the request $c_\tau$, we remove $(t,c,\sigma)$ from $\widetilde{R}$ and restore
the state associated with the shortened transcript. If
$\widetilde{R}$ becomes empty, we restore the initial state. Once this
deletion-and-rollback procedure terminates, either $\widetilde{R}$ is empty or
its last entry is not killed by $c_\tau$. We then query the Player--L strategy
at $c_\tau$ using the restored state, obtaining a label $\sigma_\tau$,
and define $
    R_\tau
    :=
    \widetilde{R}\mathbin{\|}(t_\tau,c_\tau,\sigma_\tau)$, 
where $\mathbin{\|}$ denotes concatenation. Finally, we record the state
resulting from this query as the state associated with the new reduced
transcript $R_\tau$.

Let $\operatorname{Surv}(H)$ be the number of occupied cells on the
actual board after executing the full transcript $H$.

\begin{definition}[
Reduced-execution value of the composite $A/B$ procedure]
\label{def:ValABred}
For $n,m,s\in\mathbb N$, define the horizon-restricted value
\[
    \operatorname{Val}^{\mathrm{red}}_{\mathrm{AB}}(n;m,s)
    :=
    \max_H \operatorname{Surv}(H),
\]
where the maximum is over all full legal transcripts $H$ generated by
$\mathsf{SimAB}_n$ against arbitrary adaptive choices of Player--P,
subject to $
    \abs{H}\le m$ and $
    \operatorname{rlen}(H)\le s$. 
We also define the uniform reduced-execution value
\[
    \operatorname{Val}^{\mathrm{red}}_{\mathrm{AB}}(n,s)
    :=
    \sup_{1\leq m \leq T}
    \operatorname{Val}^{\mathrm{red}}_{\mathrm{AB}}(n;m,s).
\]
\end{definition}

\begin{lemma}\label{lemma:SPR1} 
For any input sequence \(z_1,\ldots,z_T\in[0,1]\), if
\textsc{SPR-Calibration} is run with input $\{z_t\}_{t=1}^T$, then the bias term satisfies
\begin{align}
& \sum_{p\in\mathcal P_T}
\left|
\sum_{t\in [T]: p_t = p}(z_t-p_t)
\right|\le O\!\left(
\sum_{i=1}^{\tau-h}
\sum_{j= i+1}^{i+h}
2^{j-i}
\operatorname{Val}_{\mathrm{AB}}^{\mathrm{red}}(2^i,2^{\tau-j+1}) + \sum_{i=1}^{\tau-h} \sum_{j = i+1}^{i+h}\min\{ 2^i ,2^{\tau-h-i}\}
\right).\nonumber
\end{align}
\end{lemma}

\begin{proof} 

By the argument of Corollary~A.3 in~\cite{dagan2024breaking}, the reduced
transcript makes at most \(2^{\tau-j+1}\) calls to
\texttt{simulateGame} for the \textsc{SPR} instance \(G_{i,j,\ell}\). Suppose
that Player--L follows the recursive \(A/B\) labeling strategy
\(\mathsf L^{\mathrm{AB}}\). Then, by Lemmas~A.1 and~A.4 of \cite{dagan2024breaking}, together with the
definition of
\(\operatorname{Val}_{\mathrm{AB}}^{\mathrm{red}}(n,s)\), the bias contribution
generated by the instance \(G_{i,j,\ell}\) is bounded by
\[
    O\!\left(
        2^{j-i}
        \operatorname{Val}^{\mathrm{red}}_{\mathrm{AB}}
        \!\left(2^i,2^{\tau-j+1}\right)
        +
        \min\{2^i,2^{\tau-h-i}\}
    \right).
\]
Summing this bound over $
    1\le i\le \tau-h$, $
    i+1\le j\le i+h$ and 
    $\ell\in\{0,1\}$
gives the claimed bound.

\end{proof}

\begin{lemma}\label{lemma:SPR2} It holds that
\[
|\mathcal{P}_T|\leq  \sum_{i=1}^{\tau}O\left( \min \{ 2^i, 2^{\tau-h-i}\}\right)
=
O\!\left(2^{(\tau-h)/2}\right).
\]
\end{lemma}
\begin{proof}
This is the active-cell counting bound of Lemma~A.4 of
\cite{dagan2024breaking}.  Its dyadic counting argument is unaffected by the endpoint tie-breaking convention fixed above.
\end{proof}

\begin{lemma}
\label{lemma:SPR3}
There exist constants \(C_{\mathrm{AB}}\ge 1\) and
\(\alpha,\beta>0\) satisfying $
    q:=\alpha+\beta<1$ such that, for every \(n,s\ge 1\),
$$
\operatorname{Val}^{\mathrm{red}}_{\mathrm{AB}}(n,s)
    \le
    \min\left\{
        n,\,
        s,\,
        C_{\mathrm{AB}}n^\alpha s^\beta
    \right\}.$$
Consequently, defining $
    \varepsilon_{\mathrm{AB}}:=1-q>0$
and $    \gamma
    :=
    \frac{q}{1+q}
    =
    \frac{1-\varepsilon_{\mathrm{AB}}}
         {2-\varepsilon_{\mathrm{AB}}}$,
we have the uniform bound $$
\operatorname{Val}^{\mathrm{red}}_{\mathrm{AB}}(n,s)
    \le
    C_{\mathrm{AB}}^{\frac{1}{(1+q)}}
    (ns)^\gamma .$$
In particular, for every \(\rho\ge 0\),
$$
    \operatorname{Val}^{\mathrm{red}}_{\mathrm{AB}}
    \bigl(n,\ n^\rho\bigr)
    =
    O\!\left(
        n^{\gamma(1+\rho)}
    \right).$$
\end{lemma}

\begin{proof}
Fix \(n,s\ge 1\), and consider an arbitrary legal play of length
\(t\le s\).  Let \(R_\sigma\) denote the number of surviving signs of
type \(\sigma\in\{-1,+1\}\).
The root \(A\)-instance of the \(A/B\) strategy has bias parameter
zero.  Therefore, Lemma~5.1 of \cite{dagan2024breaking} gives
constants \(C_0,\alpha,\beta>0\), with \(\alpha+\beta<1\), such that $
    R_\sigma
    \le
    C_0 n^\alpha t^\beta$
for each \(\sigma\in\{-1,+1\}\).  Summing over the two sign types and
using \(t\le s\), we obtain
\[
    R_{-1}+R_{+1}
    \le
    2C_0 n^\alpha s^\beta.
\]
Thus, after setting \(C_{\mathrm{AB}}:=2C_0\),
\[
\operatorname{Val}^{\mathrm{red}}_{\mathrm{AB}}(n,s)
    \le
    C_{\mathrm{AB}}n^\alpha s^\beta.
\]

The bound $\operatorname{Val}^{\mathrm{red}}_{\mathrm{AB}}(n,s)\le \min\{n,s\}$ follows directly from the rules of the
\textsc{SPR} game. 
This proves
\[
\operatorname{Val}^{\mathrm{red}}_{\mathrm{AB}}(n,s)
    \le
    \min\left\{
        n,\,
        s,\,
        C_{\mathrm{AB}}n^\alpha s^\beta
    \right\}.
\]

Therefore, we have that
\[
\operatorname{Val}^{\mathrm{red}}_{\mathrm{AB}}(n,s)
    \le
    C_{\mathrm{AB}}^{\frac{1}{1+q}}(ns)^{\frac{q}{1+q}} =  C_{\mathrm{AB}}^{\frac{1}{1+q}}(ns)^{\gamma}.
\]
By choosing $s = n^{\rho}$, we obtain
\begin{align*}
    \operatorname{Val}_{\mathrm{AB}}^{\mathrm{red}}
    \bigl(n,n^{\rho}\bigr)
    \le
    C_{\mathrm{AB}}^{1/(1+q)}
    \left(n\cdot n^{\rho}\right)^\gamma=
    O\!\left(n^{\gamma(1+\rho)}\right).
\end{align*}

\end{proof}

By Lemma~\ref{lemma:SPR1} and Lemma~\ref{lemma:SPR3}, we have that 
\begin{lemma}\label{lemma:SPR4}
 For any input sequence \(z_1,\ldots,z_T\in[0,1]\), when
\textsc{SPR-Calibration} is executed with input \((z_t)_{t=1}^T\), the resulting bias term
satisfies
\begin{align}
 &  \sum_{p\in\mathcal P_T}
\left|
\sum_{t\in [T]: p_t = p}(z_t-p_t)
\right| \le  O\left(\sum_{i=1}^{\tau-h}\sum_{j=i+1}^{i+h} 2^{j-i + \gamma(i+\tau -j)} + \log_2(T)\cdot  \sum_{i=1}^{\tau-h} \min\{2^i, 2^{\tau-h-i}\}  \right) .\nonumber
\end{align}
\end{lemma}

%% file: SPR_alg.tex
\subsection{Deterministic Implementation of \textsc{SPR-Calibration}}
\label{sec:spralg}

This subsection specifies the version of \textsc{SPR-Calibration} used by our
outer algorithm. For completeness, we present the main algorithm of
\cite{dagan2024breaking} in Algorithm~\ref{alg:dagan-spr-calibration}, together
with its key subroutines: \texttt{simulateGame} in
Algorithm~\ref{alg:dagan-simulate}, the recursive \(A\)-procedure in
Algorithm~\ref{alg:dagan-A}, and the recursive \(B\)-procedure in
Algorithm~\ref{alg:dagan-B}.
 


The feature needed by Algorithm~\ref{alg:main} is a deterministic one-step
transition.  Given the pre-round state $S$ (see Definition~\ref{def:SPRSTATE}), a round index $t$, and an input
$z\in[0,1]$, define
\[
    \operatorname{Step}_{\mathrm{SPR}}(S,t,z)
    = (p,S^+),
\]
where $p$ is the forecast returned by one round of
\textsc{SPR-Calibration} and $S^+$ is the resulting state.  We fix all loop
orders, tie-breaking rules, and boundary conventions below, so this transition
is well-defined.  We then set
\[
\begin{split}
    \operatorname{Predict}_{\mathrm{SPR}}(S,t,z)
        &:= p,\\
    \operatorname{Update}_{\mathrm{SPR}}(S,t,z)
        &:= S^+.
\end{split}
\]


\paragraph{Scales and \textsc{SPR} instances.}
Let  $
    \tau:=\left\lceil\log_2 T\right\rceil$, 
    and let
    $\overline T:=2^\tau$.
Run the first $T$ rounds of a subroutine initialized for the padded horizon
$\overline T$.  Since $T\le \overline T<2T$, this padding changes all
asymptotic bounds by at most a constant factor.  We assume $\tau\ge2$; the
finitely many smaller horizons can be handled by any fixed forecasting rule.
Fix an integer $1\leq h <\tau$. Define $
    \mathcal I_h:=\{1,\ldots,\tau-h\}$, $
    \mathcal J_i:=\{i+1,\ldots,i+h\}$ and $
    C_i:=\{1,\ldots,2^i\}.$ 
For every $i\in\mathcal I_h$, $j\in\mathcal J_i$, and
$\ell\in\{0,1\}$, the subroutine maintains one \textsc{SPR} instance
$G_{i,j,\ell}$ with cell set $C_i$.  The index $i$ specifies the spatial
resolution, $j$ specifies the time scale, and $\ell$ separates the even and
odd dyadic intervals at resolution $2^{-(i+1)}$. 

\paragraph{Cells, intervals, and forecast values.} For $z\in[0,1]$, let $
    m_i(z)
    :=
    \min\left\{
        \left\lfloor 2^{i+1}z\right\rfloor,
        2^{i+1}-1
    \right\}$, $
    \ell_i(z)
    :=m_i(z)\bmod 2$
 and     $c_i(z)
    :=\frac{m_i(z)-\ell_i(z)}{2}+1.$ 
  We define $$
    \operatorname{cell}(i,j,z)
    :=
    \bigl(c_i(z),G_{i,j,\ell_i(z)}\bigr).$$
Thus, $j$ selects the \textsc{SPR} instance at the desired time scale, while $i$ and $z$
determine the parity and cell.

For $c\in C_i$ and $\ell\in\{0,1\}$, write $m=2(c-1)+\ell$ and define
\[
    \operatorname{interval}(c,G_{i,j,\ell})
    :=
    \begin{cases}
    \left[\dfrac{m}{2^{i+1}},\dfrac{m+1}{2^{i+1}}\right),
        &m<2^{i+1}-1,\\[2ex]
    \left[\dfrac{2^{i+1}-1}{2^{i+1}},1\right],
        &m=2^{i+1}-1.
    \end{cases}
\]
In this way, every \(z\in[0,1]\)
is assigned to a unique cell\footnote{We use this half-open interval convention to make the interval assignment explicit.}. 
For a sign $\sigma\in\{+,-\}$, define
\[
    \operatorname{prob}(c,\sigma,G_{i,j,\ell})
    :=
    \begin{cases}
    \dfrac{\max\{0,2(c-1)+\ell-1\}}{2^{i+1}},
        &\sigma=+,\\[2ex]
    \dfrac{\min\{2(c-1)+\ell+2,2^{i+1}\}}{2^{i+1}},
        &\sigma=-.
    \end{cases}
\]

\paragraph{The \textsc{SPR} state.} To formalize the subroutines that we borrow from the
\textsc{SPR-Calibration} algorithm of~\cite{dagan2024breaking}, we first give a
precise definition of the state of an \textsc{SPR-Calibration} instance.
\begin{definition}[\textsc{SPR} state]\label{def:SPRSTATE}
For each instance $G=G_{i,j,\ell}$, the subroutine stores
\[
\begin{array}{ll}
    \operatorname{bias}_G:C_i\to\mathbb R,
        &\text{the accumulated surrogate bias in each cell},\\
    \sigma_G:C_i\to\{+,-,\varnothing\},
        &\text{the current \textsc{SPR} sign configuration},\\
    \mathsf L_G,
        &\text{the mutable state of the explicit $A/B$ labeler}.
\end{array}
\]
Initially, we set $
    \operatorname{bias}_G(c)=0$ and  $   \sigma_G(c)=\varnothing$
for every instance $G$ and cell $c$.  The labeler state $\mathsf L_G$ is
initialized by $
    A_G^{\mathrm{root}}
    :=A.\operatorname{initialize}(1,2^i,0).$
    
 Accordingly, a query to calibration cell $c$ is
passed directly to the labeler as query $c$.
The complete state is
\begin{align}
    S^{\mathrm{SPR}}_{\mathrm{local}}
    =
    \left\{
        \operatorname{bias}_{G},
        \sigma_{G},
        \mathsf L_{G}
    \right\}_{
        i\in\mathcal I_h,
        j\in\mathcal J_i,
        \ell\in\{0,1\}
    }.\label{def:currentspr}
\end{align}
Moreover, to deal with the reduced transcript (see Definition~\ref{def:reduced-transcript}), we need to keep a record of all historical states as follows:
\begin{align}
 S^{\mathrm{SPR}}_{t} = \{  S^{\mathrm{SPR}}_{\mathsf{local},s}\}_{s\leq t-1} , \label{def:StSPR}
\end{align}
where $S^{\mathrm{SPR}}_{\mathsf{local},s}$ is the local state defined as \eqref{def:currentspr} at round $s$.
\end{definition}

During the learning process, the state $S_t^{\mathrm{SPR}}$ is updated to
according to the changes in $\mathrm{bias}_G$, $\sigma_G$, and $\mathsf{L}_G$ for
each instance $G = G_{i,j,\ell}$.

\paragraph{Deterministic conventions.}
All loops are executed in increasing lexicographic order of their indices.  If
more than one cell is admissible in the bias-removal phase, the smallest
admissible cell is selected.  Finally, define $
    \operatorname{sgn}(b) = +$ for $b\geq 0 $ and $ \operatorname{sgn}(b) = -$ for $b<0$. 
These conventions make the algorithm deterministic.

\paragraph{Placement of bias.} Given an input $z$, the algorithm first attempts to
reduce a previously accumulated bias of sufficiently large magnitude. If no
such bias-reduction update is available, the algorithm selects an \textsc{SPR} instance on which
the cell containing $z$ has sufficiently small current bias and adds the new
bias contribution to that cell. In particular, if the selected cell does not carry an \textsc{SPR} sign, the algorithm first invokes
\texttt{simulateGame} to place one.

\begin{algorithm}[!t]
\small
\caption{\textsc{SPR-Calibration} (Algorithm 1 in \cite{dagan2024breaking})}
\label{alg:dagan-spr-calibration}
\begin{algorithmic}[1]
\REQUIRE Pre-round state $S$, round index $t$, and input $z\in[0,1]$.

\STATE \textbf{Bias-removal phase.}
\FOR{$i\in\mathcal I_h$ in increasing order}
    \FOR{$j\in\mathcal J_i$ in increasing order}
        \STATE $(c,G)\gets\operatorname{cell}(i,j,z)$. \COMMENT{\small Locate the \textsc{SPR} instance.}
    \STATE  $\mathcal B\gets
    \{c'<c:\operatorname{bias}_G(c')<-1\}  \cup
    \{c'>c:\operatorname{bias}_G(c')>1\}$. \COMMENT{\small{Find cells with removable bias.}}
        \IF{$\mathcal B\neq\varnothing$}
            \STATE $\bar c\gets\min\mathcal B$. \COMMENT{\small Deterministic rule to rank the cells.}
\STATE
\(\left.
\begin{aligned}
&\sigma \gets \operatorname{sgn}(\operatorname{bias}_G(\bar c)),\\
&p \gets \operatorname{prob}(\bar c,\sigma,G),\\
&\operatorname{bias}_G(\bar c)
\gets \operatorname{bias}_G(\bar c)+(z-p).
\end{aligned}
\right\}\)
\COMMENT{\small Update the selected \textsc{SPR} instance \(G\).}
\STATE Update \(S\) according to~\eqref{def:StSPR}.
\COMMENT{\small Update the state $S$ after updating the \textsc{SPR} instance \(G\).}
            \STATE \textbf{return} $(p,S)$.\label{line:return}
        \ENDIF
    \ENDFOR
\ENDFOR

\STATE \textbf{Bias-placement phase.}
\FOR{$i\in\mathcal I_h$ in increasing order}
    \FOR{$j\in\mathcal J_i$ in increasing order}
        \STATE $(c,G)\gets\operatorname{cell}(i,j,z)$. \COMMENT{\small Locate the \textsc{SPR} instance.}
        \IF{$|\operatorname{bias}_G(c)|<2^{j-i}$}
            \IF{$\sigma_G(c)=\varnothing$}
\STATE \texttt{simulateGame}\((c,G)\). \COMMENT{\small{Call the \textsc{SPR} instance; the  labeling procedure \(\mathsf L_G\) is updated.}}
            \ENDIF
            \STATE \(\left.
\begin{aligned}
&\sigma\gets\sigma_G(c) ,
\\ & p\gets\operatorname{prob}(c,\sigma,G),
\\ & \operatorname{bias}_G(c)
            \gets\operatorname{bias}_G(c)+(z-p).
\end{aligned}
\right\}\)\COMMENT{\small Update the selected \textsc{SPR} instance \(G\).}
           \STATE Update \(S\) according to~\eqref{def:StSPR}.
\COMMENT{\small Update the state $S$ after updating the \textsc{SPR} instance \(G\).}
            \STATE \textbf{return} $(p,S)$.\label{line:return1}
        \ENDIF
    \ENDFOR
\ENDFOR
\end{algorithmic}
\end{algorithm}

\paragraph{The \texttt{simulateGame} subroutine.}
A call to \texttt{simulateGame} performs one legal move of the \textsc{SPR} game on an \textsc{SPR} instance whose queried cell is empty.  It removes every minus sign strictly to the
left of the queried cell and every plus sign strictly to its right, asks the
explicit $A/B$ strategy for the new sign, and records that sign in the queried
cell.

\begin{algorithm}[!t]
\small
\caption{\texttt{simulateGame}$(c,G)$}
\label{alg:dagan-simulate}
\begin{algorithmic}[1]
\REQUIRE An empty cell \(c\in C_i=[2^i]\) of an \textsc{SPR} instance
\(G=G_{i,j,\ell}\).
\STATE \COMMENT{\small Remove the $-$ signs left to $c$.}
\FOR{each \(c'\in C_i\) with \(c'<c\)} 
    \IF{\(\sigma_G(c')=-\)}
        \STATE \(\sigma_G(c')\gets\varnothing\).
    \ENDIF
\ENDFOR
\STATE\COMMENT{\small Remove the $+$ signs right to $c$.}
\FOR{each \(c'\in C_i\) with \(c'>c\)}  
    \IF{\(\sigma_G(c')=+\)}
        \STATE \(\sigma_G(c')\gets\varnothing\).
    \ENDIF
\ENDFOR
\STATE \(s\gets\mathsf L_G.\operatorname{label}(c)\). \COMMENT{\small Run the \textsc{SPR} procedure to decide the sign of $c$.}
\STATE \(\sigma_G(c)\gets s\).
\end{algorithmic}
\end{algorithm}

\paragraph{The explicit $A/B$ Player--L strategy.}
Each \textsc{SPR} instance uses the deterministic Player--L strategy of
\cite{dagan2024breaking}.  It is defined by two mutually recursive procedures, $A$ and $B$.  An $A$-object controls an interval $[l,r]$ and an integer
bias parameter $b$.  At a leaf it returns $\operatorname{sgn}(b)$.  At an
internal node it delegates to its currently active $B$-object and restarts that
object whenever the latter returns the restart symbol $\perp$.

\begin{algorithm}[!t]
\small
\caption{Player--L procedure $A$ (Algorithm 3 in \cite{dagan2024breaking})}
\label{alg:dagan-A}
\begin{algorithmic}[1]
\REQUIRE Integers $1\le l\le r$ and bias parameter $b\in\mathbb Z$.

\STATE \textbf{Procedure} $A.\operatorname{initialize}(l,r,b)$:
\IF{$l<r$}
    \STATE $\operatorname{recentB}
    \gets B.\operatorname{initialize}(l,r,b,1)$.
\ELSE
    \STATE $\operatorname{recentB}\gets\varnothing$.
\ENDIF
\STATE $\operatorname{count}\gets0$.

\STATE \textbf{Procedure} $A.\operatorname{label}(s)$, where $s\in[l,r]$:
\IF{$l=r$}
    \STATE \textbf{return} $\operatorname{sgn}(b)$.
\ELSE
    \STATE $\operatorname{count}\gets\operatorname{count}+1$.
    \STATE $\sigma\gets\operatorname{recentB}.\operatorname{label}(s)$.
    \IF{$\sigma=\perp$}
        \STATE $\operatorname{recentB}
        \gets B.\operatorname{initialize}(l,r,b,\operatorname{count})$.
        \STATE $\operatorname{count}\gets1$.
        \STATE \textbf{return}
        $\operatorname{recentB}.\operatorname{label}(s)$.
    \ELSE
        \STATE \textbf{return} $\sigma$.
    \ENDIF
\ENDIF
\end{algorithmic}
\end{algorithm}

A $B$-object splits its interval into two halves, stores one child $A$-object
for each half, and uses four phases to compare how often the two halves have
been queried.  The parameter $M$ is a running guess for the relevant execution
length and implements a doubling mechanism. 
All counters, phases, child
objects, and pointers are part of the labeler state $\mathsf L_G$.  Thus, after
$\mathsf L_G$ is included in $S_t^{\mathrm{SPR}}$, a call to
\texttt{simulateGame} is a deterministic function of the copied state and the
queried cell.

\begin{algorithm}[!t]
\small
\caption{Player--L procedure $B$ (Algorithm 4 in \cite{dagan2024breaking})}
\label{alg:dagan-B}
\begin{algorithmic}[1]
\REQUIRE Integers $1\le l<r$, bias parameter $b\in\mathbb Z$, and guess
parameter $M\in\mathbb N$.

\STATE \textbf{Procedure} $B.\operatorname{initialize}(l,r,b,M)$:
\STATE $m\gets\lfloor(l+r)/2\rfloor$.
\STATE $A[0]\gets A.\operatorname{initialize}(l,m,b)$.
\STATE $A[1]\gets A.\operatorname{initialize}(m+1,r,b)$.
\STATE $\operatorname{prevHalf}\gets-1$.
\STATE $\operatorname{countHalf}[0]\gets0$,
$\operatorname{countHalf}[1]\gets0$.
\STATE $\operatorname{phase}\gets1$.
\STATE \textbf{Procedure} $B.\operatorname{label}(s)$, where $s\in[l,r]$:
\STATE $\operatorname{half}\gets0$ if $s\le m$, and
$\operatorname{half}\gets1$ otherwise.
\STATE $\operatorname{countHalf}[\operatorname{half}]
\gets\operatorname{countHalf}[\operatorname{half}]+1$.

\IF{$\operatorname{phase}=1$}
    \IF{$\operatorname{countHalf}[\operatorname{half}]=M$
    and
    $M\le\operatorname{countHalf}[1-\operatorname{half}]\le2M$}
        \STATE $\operatorname{phase}\gets2$.
    \ELSIF{$\operatorname{countHalf}[\operatorname{half}]=M$
    and
    $\operatorname{countHalf}[1-\operatorname{half}]>2M$}
        \STATE $\operatorname{phase}\gets3$.
    \ENDIF

\ELSIF{$\operatorname{phase}=2$}
    \IF{$\operatorname{half}\neq\operatorname{prevHalf}$}
        \STATE \textbf{return} $\perp$.
    \ENDIF
    \IF{$\operatorname{countHalf}[\operatorname{half}]
    =2\operatorname{countHalf}[1-\operatorname{half}]+1$}
        \STATE $\operatorname{phase}\gets3$.
    \ENDIF

\ELSIF{$\operatorname{phase}=3$}
    \IF{$\operatorname{countHalf}[\operatorname{half}]
    =\left\lfloor
    \operatorname{countHalf}[1-\operatorname{half}]/2
    \right\rfloor+1$}
        \STATE $\operatorname{phase}\gets4$.
        \IF{$\operatorname{half}=0$}
            \STATE $A[0]\gets A.\operatorname{initialize}(l,m,b+1)$.
        \ELSE
            \STATE $A[1]\gets A.\operatorname{initialize}(m+1,r,b-1)$.
        \ENDIF
    \ENDIF

\ELSIF{$\operatorname{phase}=4$}
    \IF{$\operatorname{countHalf}[\operatorname{half}]
    >\operatorname{countHalf}[1-\operatorname{half}]$}
        \STATE \textbf{return} $\perp$.
    \ENDIF
\ENDIF

\STATE $\sigma\gets A[\operatorname{half}].\operatorname{label}(s)$.
\STATE $\operatorname{prevHalf}\gets\operatorname{half}$.
\STATE \textbf{return} $\sigma$.
\end{algorithmic}
\end{algorithm}

%% file: algorithm.tex
\section{Algorithm}
\label{sec:alg}

We present our main algorithm in Algorithm~\ref{alg:main}. The algorithm
maintains an instance of \textsc{SPR-Calibration}. The main difficulty in
applying \textsc{SPR-Calibration} is that the conditional means
\(\{e_t\}_{t=1}^T\) are not available to the forecaster. To address this, we
use a Blackwell-style correction layer to construct, at each round
\(t\in[T]\), a surrogate estimate \(\widetilde e_t\) of \(e_t\). This surrogate
is then passed as the input to the \textsc{SPR-Calibration} instance. 

 Let $\mathcal X  =\{0, \eta, 2\eta,\ldots, N\eta \}$ be a grid with threshold $\eta = T^{-1/2}$ and  $N =\left\lfloor 1/\eta\right\rfloor$.

At each round $t$, let \(S_t^{\mathrm{SPR}}\) denote the internal state of the
\textsc{SPR-Calibration} instance used by Algorithm~\ref{alg:main} 
before round \(t\) and $\mathcal{P}_t = \{ p_s: s\leq t\}$.  For each $x\in \mathcal{X}$ and $p\in \mathcal{P}_{t-1}$, we define

\begin{align}
 D_t(x)
     & :=  \operatorname{Predict}_{\mathrm{SPR}}( 
    S^{\mathrm{SPR}}_t,t ,x)\label{eq:defdt}
    \\ R_{p, t-1} & : = \sum_{s<t: p_s =p} (y_s - \tilde{e}_s)\label{eq:defR}
    \end{align}

Moreover, for each \(x\in\mathcal X\), define $
\widehat R_{x,t-1}:=R_{D_t(x),t-1}.$ 
In particular, if \(D_t(x)\notin\mathcal P_{t-1}\),  we use the default initialization $
\widehat R_{x,t-1}
=
R_{D_t(x),t-1}
=
0.$

Given the quantities \(\widehat R_{x,t-1}\) defined above, we solve the following
Blackwell-style minimax program to obtain an optimal distribution
\(\mu_t\in\Delta(\mathcal X)\), where \(\Delta(\mathcal X)\) denotes the set of
probability distributions over \(\mathcal X\).
\begin{align}
\beta_t
=
\min_{\mu\in\Delta(\mathcal X)}
\max_{y\in\{0,1\}}
\sum_{x\in\mathcal X}
\mu(x)\widehat{R}_{x, t-1}(y-x).\label{eq:opt}
\end{align}

Then the algorithm samples the surrogate mean \(\widetilde e_t\sim \mu_t\), and
commits the deterministic \textsc{SPR-Calibration} steps:
\begin{align}
p_t
 = D_t(\widetilde e_t) & =
\operatorname{Predict}_{\mathrm{SPR}}
(S_t^{\mathrm{SPR}},t,\widetilde e_t), \nonumber
\\S_{t+1}^{\mathrm{SPR}}
 & =
\operatorname{Update}_{\mathrm{SPR}}
(S_t^{\mathrm{SPR}},t,\widetilde e_t).\nonumber
\end{align}

After outputting the forecast \(p_t\) and observing the true outcome \(y_t\), we update $\mathcal{P}_t$ and
the residuals \(R_{p,t}\) for all \(p\in\mathcal P_t\).

\begin{algorithm}[!t]
\small
\caption{Blackwell-Wrapped \textsc{SPR-Calibration}}
\label{alg:main}
\begin{algorithmic}[1]
\REQUIRE Horizon $T$, grid $\mathcal X$.
\STATE Initialize the internal state $S^{\mathrm{SPR}}_1$ of \textsc{SPR-Calibration}
\STATE Initialize the set of actually predicted values $\mathcal P_0\gets\emptyset$

\FOR{$t=1,2,\ldots,T$}

    \STATE Define $D_t(x)$ following \eqref{eq:defdt} for all $x\in \mathcal{X}$;
    \STATE For all $x\in \mathcal{X}$, set $\widehat{R}_{x,t-1} = R_{D_t(x),t-1}$ .
    \STATE Choose $\mu_t\in\Delta(\mathcal X)$ by solving the finite minimax optimization
    $$\min_{\mu\in\Delta(\mathcal X)}
\max_{y\in\{0,1\}}
\sum_{x\in\mathcal X}
\mu(x)\widehat{R}_{x,t-1}(y-x)$$

    \STATE Sample an internal surrogate mean $
    \widetilde e_t\sim\mu_t.$

    \STATE Feed $\widetilde e_t$ to \textsc{SPR-Calibration} and update \textsc{SPR-Calibration}:
    \begin{align}
    (p_t, S^{\mathrm{SPR}}_{t+1})
    \gets
    \left(\operatorname{Predict}_{\mathrm{SPR}}(S^{\mathrm{SPR}}_t,t,\widetilde e_t) , \operatorname{Update}_{\mathrm{SPR}}(S^{\mathrm{SPR}}_t,t,\widetilde e_t)\right).\label{eq:update1}
    \end{align}
    \IF{$p_t\notin\mathcal P_{t-1}$}
        \STATE $\mathcal P_t\gets \mathcal P_{t-1}\cup\{p_t\}$
        \STATE   $R_{p_t, t-1}\gets 0$
    \ELSE
        \STATE $\mathcal P_t\gets\mathcal P_{t-1}$
    \ENDIF
    \STATE Observe the outcome $y_t\in\{0,1\}$.
    \STATE Update the external Blackwell residual: $
    R_{p_t,t}
    \gets
    R_{p_t,t-1}+y_t-\widetilde e_t$.

    \STATE For all $p\in\mathcal P_t\setminus\{p_t\}$, set $
    R_{p,t}\gets R_{p,t-1}$.

\ENDFOR
\end{algorithmic}
\end{algorithm}

%% file: analysis.tex
\section{Analysis}\label{sec:analysis}

In this section, we present the proof of Theorem~\ref{thm:main}. The proof consists of two parts: bounding the calibration error and the computational cost.

\subsection{Proof of Calibration Error}

Set the hyperparameters of \textsc{SPR-Calibration} as
\(\tau=\left\lceil \log_2 T\right\rceil\), and choose the scale parameter
\(h\leq \tau\) later.
Recall  $\mathcal{P}_t=\{p_s:s\le t\}$. For  $t\in [T]$ and  $ p\in \mathcal{P}_t$, recall the outer residual $
R_{p,t}
=
\sum_{s\le t:p_s=p}(y_s-\widetilde e_s)$ if $p\in \mathcal{P}_t$. Define $R_{p,t} = 0$ if $p\notin \mathcal{P}_t$.

For every $p\in\mathcal P_T$,
\[
\sum_{t\in [T]: p_t = p}(y_t-p)
=
\sum_{t\in [T]: p_t = p}(\widetilde e_t-p)
+
\sum_{t\in [T]: p_t = p}(y_t-\widetilde e_t).
\]
Therefore,
\begin{align}
\mathbb{E}[\operatorname{CalErr}_T]
\le
\underbrace{\mathbb{E}\left[
\sum_{p\in\mathcal P_T}
\left|
\sum_{t\in [T]: p_t = p}(\widetilde e_t-p)
\right|\right]
}_{\text{\textsc{SPR-Calibration} bias term}}
+
\underbrace{
\mathbb{E}\left[\sum_{p\in\mathcal P_T}|R_{p,T}|\right]
}_{\text{Blackwell residual term}}.
\end{align}

\paragraph{Bound of the \textsc{SPR-Calibration} bias term.} 
The bound of the \textsc{SPR-Calibration} bias  directly follows the result in \cite{dagan2024breaking}. 
By  Lemma~\ref{lemma:SPR4}, we have that 
\begin{align}
\mathbb{E}\left[
\sum_{p\in\mathcal P_T}
\left|
\sum_{t\in [T]: p_t = p}(\widetilde e_t-p)
\right|\right] \leq O\!\left(
\sum_{i=1}^{\tau-h}
\sum_{j=i+1}^{i+h}
2^{j-i}
2^{\gamma(i+\tau-j)} + \sum_{i\leq \tau - h} \sum_{j\leq i+h}\min\{ 2^i ,2^{\tau-h-i}\}
\right).\label{eq:boundB}
\end{align}

\paragraph{Bound of the Blackwell residual term.}

By the property of \textsc{SPR-Calibration} (see Lemma~\ref{lemma:SPR2}), we have that $
|\mathcal P_T|\le O\!\left(2^{(\tau-h)/2}\right).$
Then by Cauchy's inequality, we have that 
\begin{align}
\mathbb{E}\left[\sum_{p\in \mathcal{P}_T} \left| R_{p,T} \right|\right]\leq O\left( \sqrt{ 2^{(\tau-h)/2} \mathbb{E}\left[\sum_{p\in \mathcal{P}_T}R_{p,T}^2\right]} \right)\leq  O\left( \sqrt{  \mathbb{E}[\Phi_T]}\cdot 2^{(\tau-h)/4}\right),\label{eq:boundP1}
\end{align}
where we define $\Phi_t =\sum_{p\in \mathcal{P}_t}R_{p,t}^2 $ for $t\in [T]$ and $\Phi_0 = 0$. Then we have the following lemma to bound $\mathbb{E}[\Phi_T]$.

\begin{lemma}\label{lemma:bdpht} $\mathbb{E}[\Phi_T] = O\!\left((\sqrt T+\eta T)^2\right)$. 
If $\eta = O( T^{-1/2})$, then $\mathbb{E}[\Phi_T] = O(T)$.
\end{lemma}

By Lemma~\ref{lemma:bdpht} and \eqref{eq:boundP1}, we have that $\mathbb{E}\left[\sum_{p\in \mathcal{P}_T} \left| R_{p,T} \right|\right]\leq O(\sqrt{T}\cdot 2^{(\tau-h)/4})$.

\paragraph{Putting all together.}
Recall that 
\(\eta = T^{-1/2}\). Let $\varepsilon_{\mathrm{AB}}$ be defined in Lemma~\ref{lemma:SPR3} and let $\gamma  = \frac{1-\varepsilon_{\mathrm{AB}} }{2- \varepsilon_{\mathrm{AB}}}$.
We then have that 
\begin{align}
& \mathbb E[\operatorname{CalErr}_T] \nonumber
\\ & \leq  O\!\left(
\sum_{i=1}^{\tau-h}
\sum_{j=i+1}^{i+h}
2^{j-i}
2^{\gamma(i+\tau-j)} + \sum_{i=1}^{\tau-h} \sum_{j\leq i+h}\min\{ 2^i ,2^{\tau-h-i}\}
\right) + O\left( \sqrt{T}\cdot 2^{(\tau-h)/4}\right) \nonumber
\\ &  \leq   O\left(\sum_{i=1}^{\tau-h} 2^{h(1-\gamma) +  \gamma\tau } + \log_2(T)\cdot  \sum_{i\leq \tau -h} \min\{2^i, 2^{\tau-h-i}\}  +\sqrt{T}\cdot 2^{(\tau-h)/4}  \right)  \nonumber
\\ & \leq O\left( \log_2(T) \cdot 2^{h+\gamma(\tau -h)} + \log_2(T) \cdot 2^{(\tau-h)/2}  + \sqrt{T} \cdot 2^{(\tau-h)/4}\right) \label{eq:exp1}
\\ & \leq O\!\left(\log_2(T)\cdot 
2^{h+\gamma(\tau-h)}+\log_2(T)\cdot 2^{3\tau/4-h/4}
\right).\label{eq:exp2}
\end{align}
Here, \eqref{eq:exp1} follows from the inequality \(j-i\le h\) and the fact that  $\sum_{i=1}^{\tau-h} \min\{ 2^i, 2^{\tau-h-i}\} = O(2^{(\tau-h)/2})$, while
\eqref{eq:exp2} follows from \((\tau-h)/2\le 3\tau/4-h/4\).

By choosing $h = \frac{3-4\gamma}{5-4\gamma}\tau = \frac{2+\varepsilon_{\mathrm{AB}}}{6-\varepsilon_{\mathrm{AB}}} \tau $, we have that $\mathbb E[\operatorname{CalErr}_T]\leq O\left(\log_2(T)\cdot  T^{\frac{2}{3} - \frac{\varepsilon_{\mathrm{AB}}}{18}}\right)$.

\subsection{Computational cost}
\label{sec:computational-cost}

We give a fully explicit implementation of Algorithm~\ref{alg:main}. The
purpose of the argument is only to establish a polynomial running time, so we
use direct scans and deep copies rather than more sophisticated persistent data
structures.
Let $K = \left\lceil 1/\eta \right\rceil$. Recall also that $
    \tau=\left\lceil\log_2T\right\rceil$,  $\overline T=2^\tau<2T$
 and   $h\le \tau-1$. 
 

\begin{lemma}
\label{lemma:computational-cost}
Algorithm~\ref{alg:main}, using the deterministic transition from
Section~\ref{sec:spralg} and the explicit \(A/B\) labeler, can be implemented
in time $
   O\!\left(T^{\frac{7}{2}}\log T\right)$.
\end{lemma}

\begin{proof}
We prove by bounding  the cost of
one deterministic \textsc{SPR} transition and the additional work performed by the outer wrapper.

\paragraph{Size of the \textsc{SPR} state.}
For every $
    i\in\mathcal I_h=\{1,\ldots,\tau-h\}$, $
    j\in\mathcal J_i=\{i+1,\ldots,i+h\}$ and $
    \ell\in\{0,1\}$, 
the algorithm maintains a board \(G_{i,j,\ell}\) with \(2^i\) cells. Its bias
array and sign array use \(O(2^i)\) words.

The live state of the explicit \(A/B\) labeler also uses \(O(2^i)\) words.
Indeed, its recursively stored objects form a binary tree with \(O(2^i)\)
nodes, and every node stores only a constant number of counters, pointers,
phase variables, and integer parameters. When a labeler object is restarted,
the old object is discarded, so obsolete versions are not retained.

Thus, one board requires \(O(2^i)\) words, and the local \textsc{SPR} state has size
\begin{align*}
    \sum_{i=1}^{\tau-h}
    \sum_{j=i+1}^{i+h}
    \sum_{\ell\in\{0,1\}}
    O(2^i)
    &=
    O\!\left(
        h\sum_{i=1}^{\tau-h}2^i
    \right)=
    O\!\left(
        h2^{\tau-h}
    \right).
\end{align*}
Taking the historical recording of the local \textsc{SPR} states into consideration, the complete \textsc{SPR} state has size $O(Th 2^{\tau-h})$. 
The same bound applies to the one-time initialization cost and to the cost of
making a deep copy of the state.

\paragraph{Cost of one \textsc{SPR} transition.}
Consider $
    \operatorname{Step}_{\mathrm{SPR}}(S,t,z)$
for a fixed input \(z\in[0,1]\).
In the bias-removal phase, for each pair \((i,j)\), the parity of \(z\)
selects one of the two boards \(G_{i,j,0}\) and \(G_{i,j,1}\). A direct scan
of that board finds an admissible cell in \(O(2^i)\) time if one exists.
Hence a complete bias-removal scan costs $
    O\!\left(
        h\sum_{i=1}^{\tau-h}2^i
    \right)
    =
    O\!\left(
        h2^{\tau-h}
    \right).$
The bias-placement phase performs only \(O(h\tau)\) constant-time
current-cell checks, except that it may make one call to
\texttt{simulateGame}.
Suppose that \texttt{simulateGame} is called on a board with \(2^i\) cells.
Scanning the sign array, constructing the reduced transcripts, and  carrying out all else legal deletions costs \(O(2^i)\), 

It remains to bound the cost of the \(A/B\) label query.
A call to the recursive \(A/B\) procedure follows a root-to-leaf path of depth
at most \(i\). At each depth, the corresponding \(A\)-object can restart its
current \(B\)-object at most once before returning a sign. Eagerly initializing
a \(B\)-object on an interval of length at most \(2^i\) costs \(O(2^i)\).
Therefore, the deliberately coarse bound $
    O(i2^i)
    \le
    O\!\left(
        \tau2^{\tau-h}
    \right)$
holds for one label query.

Combining the bias-removal phase, the bias-placement phase, and the possible
game simulation gives that the time cost of $\operatorname{Step}_{\mathrm{SPR}}(S,t,z)$ is bounded by $O(\tau 2^{\tau-h})$.

\paragraph{Computing all preview predictions.}
For each \(x\in\mathcal X\), the
algorithm:

\begin{enumerate}
    \item makes a deep copy of \(S_t^{\mathrm{SPR}}\);
    \item evaluates $
        \operatorname{Step}_{\mathrm{SPR}}
        (S_t^{\mathrm{SPR}},t,x)$
    on the scratch copy;
    \item records only the resulting prediction \(D_t(x)\); and
    \item discards the scratch copy.
\end{enumerate}

Consequently, computing all \(K\) values \(D_t(x)\) costs $
    O\!\left(
        TK\tau2^{\tau-h}
    \right)$
time. Because the scratch state is reused, this requires only one scratch copy,
rather than \(K\) simultaneous copies.

\paragraph{Solving the finite minimax problem.} The minimax optimization problem~\eqref{eq:opt} has $K$ variables, corresponding
to the distribution over the $K$ grid points, and only two constraints,
corresponding to the two possible outcomes \(y=0\) and \(y=1\). Therefore,  it can be
solved within $O(K^2)$ time.

\paragraph{Putting all together.}
After sampling \(\widetilde e_t\), the algorithm commits exactly one SPR
transition:
$$
    (p_t,S_{t+1}^{\mathrm{SPR}})
    =
    \operatorname{Step}_{\mathrm{SPR}}
    (S_t^{\mathrm{SPR}},t,\widetilde e_t).
$$ 
This costs $
    O\!\left(
        \tau2^{\tau-h}
    \right)$. 
Therefore, the total cost of round \(t\) is $
    O\!\left(
       TK\tau2^{\tau-h}+K^2
    \right).$
Summing over the $T$ rounds and including the one-time initialization cost
gives $
    O\!\left(
        T^2K\tau2^{\tau-h}+TK^2
    \right).$ 
Using $
    K=O(\sqrt T)$, $
    \tau=O(\log T)$,
    and $
    2^{\tau-h}\le 2^\tau<2T$,
we obtain
\begin{align}
   O\left( T^2K\tau2^{\tau-h}+TK^2\right) = 
    O\!\left(
        T^{\frac{7}{2}}\log T
    \right).\nonumber
\end{align}

\end{proof}

\subsection{Proof of Lemma~\ref{lemma:bdpht}}\label{sec:missing_proof}

\textbf{Lemma~\ref{lemma:bdpht}} (restatement). \emph{$\mathbb{E}[\Phi_T] = O\!\left((\sqrt T+\eta T)^2\right)$. 
If $\eta = O( T^{-1/2})$, then $\mathbb{E}[\Phi_T] = O(T)$.}

\begin{proof}
Recall that
$\mathcal{X}$ is the grid with interval length $\eta$. 
At round $t$, the wrapper chooses a distribution $\mu_t\in\Delta(\mathcal X)$
by solving the finite minimax problem \eqref{eq:opt} as
\begin{align}
\beta_t
=
\min_{\mu\in\Delta(\mathcal X)}
\max_{y\in\{0,1\}}
\sum_{x\in\mathcal X}
\mu(x)\widehat{R}_{x, t-1}(y-x).\nonumber
\end{align}

By Von Neumann's minimax theorem 
\begin{align}
\beta_t
 & =
\max_{q\in\Delta(\{0,1\})}
\min_{x\in\mathcal X}
\mathbb{E}_{y\sim q}\left[\widehat{R}_{x, t-1}(y-x)\right] \nonumber
\\ & = \max_{q\in\Delta(\{0,1\})}
\min_{x\in\mathcal X}
\widehat{R}_{x, t-1}(\mathbb{E}_{y\sim q}[y]-x) \nonumber
\\ &= \max_{y\in [0,1]}
\min_{x\in\mathcal X}
\widehat{R}_{x, t-1}(y-x). \label{eq:optdual}
\end{align}
Fix any $y\in[0,1]$, and let $
x^*(y)\in\arg\min_{x\in\mathcal X}|y-x|$.
Then $
|y-x^*(y)|\le\eta.$
Therefore,
\[
\min_{x\in \mathcal{X}}\widehat{R}_{x, t-1}(y-x)
\le
\left|\widehat{R}_{x^*(y), t-1}\right|
\eta.
\]
Since $
\left|\widehat{R}_{x^*(y), t-1}\right|
\le
\sqrt{\Phi_{t-1}}$,
we have $
\min_{x\in\mathcal X}
\widehat{R}_{x^*(y), t-1}(y-x)
\le
\eta\sqrt{\Phi_{t-1}}.$ 
Taking the maximum over $y\in[0,1]$ gives $
\beta_t\le \eta\sqrt{\Phi_{t-1}}.$
Since $\mu_t$ is optimal for the finite minimax problem, for both
$y\in\{0,1\}$, $
\mathbb E_{x\sim\mu_t}
\left[
\widehat{R}_{x, t-1}(y-x)
\right]
\le
\beta_t
\le
\eta\sqrt{\Phi_{t-1}}$. 
Recalling that $\widehat{R}_{x,t-1} = R_{D_t(x), t-1}$ for all $x\in \mathcal{X}$, we always have 
\begin{align}
\mathbb E_{x\sim\mu_t}
\left[
R_{D_t(x), t-1}(y_t-x)
\right]
\le
\beta_t
\le
\eta\sqrt{\Phi_{t-1}}.\label{eq:us1}
\end{align}

Let $\mathcal{F}_{t}$ denote the $\sigma$-field after the $t$-th round. So the adversary chooses $y_t$ (or its distribution) conditioned on $\mathcal{F}_t$.
We then have that 
\begin{align}
\mathbb E[\Phi_t-\Phi_{t-1}\mid\mathcal F_{t-1}] = \mathbb{E}_{x\sim \mu_t}[2R_{D_t(x), t-1}(y_t-x)+(y_t-x)^2 \mid \mathcal{F}_{t-1}]\leq 2\eta \sqrt{\Phi_{t-1}} + 1,\nonumber
\end{align}
which implies
\begin{align}
\mathbb{E}[\Phi_t] - \mathbb{E}[\Phi_{t-1}]\leq 2\eta \mathbb{E}[\sqrt{\Phi_{t-1}}] + 1.\nonumber
\end{align}
Let $
Z_t=\mathbb E[\Phi_t]$ for $t\geq 0$.
By Jensen's inequality, $
\mathbb E[\sqrt{\Phi_{t-1}}]\le \sqrt{Z_{t-1}}.$
Therefore,
\[
Z_t
\le
Z_{t-1}+1+2\eta\sqrt{Z_{t-1}}.
\]
Summing over $t$ yields $
Z_T
\le
T+2\eta\sum_{t=0}^{T-1}\sqrt{Z_t}$.
Let $
M_T=\max_{0\le t\le T}\sqrt{Z_t}.$
Then $
M_T^2
\le
T+2\eta T M_T.$
Solving this quadratic inequality gives
\[
M_T
\le
\eta T+\sqrt{\eta^2T^2+T}
\le
2\eta T+\sqrt T.
\]
Hence
\[Z_T
\le
M_T^2
\le
O\!\left((\sqrt T+\eta T)^2\right).
\]
If $\eta\le T^{-1/2}$, then $
\sqrt T+\eta T=O(\sqrt T)$,
and therefore $
\mathbb E[\Phi_T]=O(T)$.
\end{proof}

%% file: discussion.tex
\section{Discussion}\label{sec:discussion}

In this paper, we develop an efficient algorithm that achieves
\(O(T^{2/3-\varepsilon})\) calibration error. Our algorithm is based on a simple
combination of the \textsc{SPR-Calibration} algorithm in \cite{dagan2024breaking} and a Blackwell-approachability correction
argument. The resulting improvement exponent \(\varepsilon\) is inherited from
the \textsc{SPR} guarantee in~\cite{dagan2024breaking}; in particular, any improvement in
the value of the underlying \textsc{SPR} game would translate directly into a stronger efficient calibration bound.